\documentclass[12pt]{article}
\usepackage[margin=0.9in]{geometry}

\usepackage[round]{natbib}
\usepackage{amsmath,amssymb,amsthm}
\usepackage{bm}

\usepackage{shortcuts}

\usepackage{algorithm}
\usepackage{algorithmic}
\usepackage{thm-restate}
\usepackage[titlenumbered,linesnumbered,ruled,noend,algo2e]{algorithm2e}

\SetCommentSty{mycommfont}
\SetEndCharOfAlgoLine{}

\usepackage[textsize=tiny]{todonotes}
\newcommand{\OT}{\mathrm{OT}}

\newcommand{\simplex}{\Delta}

\title{A note on the relations between mixture models, maximum-likelihood and entropic optimal transport}
\author{Titouan Vayer \\
  Inria, ENS de Lyon, CNRS, UCBL, LIP, UMR 5668. \\
  \texttt{titouan.vayer@inria.fr} \\
  \and Etienne Lasalle \\
  ENS de Lyon, CNRS, UCBL, Inria, LIP, UMR 5668. \\
  \texttt{etienne.lasalle@ens-lyon.fr}}

\begin{document}
\maketitle

\begin{abstract}
This note aims to demonstrate that performing maximum-likelihood estimation for a mixture model is equivalent to minimizing over the parameters an optimal transport problem with entropic regularization. The objective is pedagogical: we seek to present this already known result in a concise and hopefully simple manner. We give an illustration with Gaussian mixture models by showing that the standard EM algorithm is a specific block-coordinate descent on an optimal transport loss.
\end{abstract}

\paragraph{Notations.} Any vector $\xbf \in \R^{K}$ is treated as column matrix. The discrete probability simplex with $K$ bins is noted as $\simplex_K = \{\abf \in \R^{K}_{+}: \sum_{j=1}^{K} a_j = 1\}$. The vector of $K$ ones is denoted as $\one_K$. $\delta_x$ is the dirac mass supported at $x$. For simplicity in this note, we will deliberately remain vague in certain places regarding ‘‘edge cases'' like $0\log 0$. To be entirely rigorous, we would need to reason on the supports of the matrices and define objects that can take $\infty$ as a value.

\section{Introduction and preliminaries on optimal transport}

The relations between maximum-likelihood and optimal transport (OT) have already been discussed in multiple works \citep{rigollet2018entropic, mena2020sinkhorn,diebold2024unifiedframeworkhardsoft}. The purpose of this brief note is to provide the key tools used to establish these connections.
The primary aim is pedagogical: we will focus on the (discrete) mixtures case, adopting a ‘‘computational OT'' perspective.
Hopefully, readers will find this exercise insightful.
Our analysis will largely rely on the approach described in \citet{rigollet2018entropic}, though adapted to a different formalism and applied to a slightly different problem (mixture estimation rather than Gaussian deconvolution).

To fix the notations, we first briefly recall the fundamentals of (entropic) discrete optimal transport (OT) and readers seeking more details can refer to \citet{peyre2019computational}.
Let $\Cbf \in \R^{n \times K}$ be a cost matrix, representing, for example, the distances between points from two distributions. Let $\abf \in \simplex_n, \bbf \in \simplex_K$ be two probability vectors, encoding for instance the quantities to be transported and supplied, respectively.
The goal of OT is to determine a way to move these quantities while respecting supply constraints, such that the total transportation cost is minimized (defined by $\Cbf$).
This is formalized through the set of couplings, or transport plans, with marginals $\abf, \bbf$, which is defined by
\begin{equation}
U(\abf, \bbf) \triangleq \{\Pbf \in \R_{+}^{n \times K}:  \Pbf \one_K = \abf, \Pbf^\top \one_n = \bbf\}\,.
\end{equation}
When $\Pbf \in U(\abf, \bbf), P_{ij} \in [0,1]$ represents the of probability mass transported from the $i$-th point to the $j$-th.
When we only care about transporting the mass of the input measure without constraint on the supply we can consider \emph{semi-relaxed} transport plans as
\begin{equation}
U_{K}(\abf) \triangleq \{\Pbf \in \R_{+}^{n \times K}:  \Pbf \one_K = \abf\}\,.
\end{equation}
Given a transport plan $\Pbf$, the total cost of displacement is given by $\langle \Cbf, \Pbf \rangle$ and the goal of standard OT is to find the transport plan that minimizes this cost.

Entropic regularization, introduced in \citet{cuturi2013sinkhorn}, was proposed to accelerate the computation of the optimal transport plan.
It relies on the Kullback-Leibler divergence between two matrices $\Pbf, \Qbf \in \R_+^{n \times K}$ which is defined as
\begin{equation}
    \KL(\Pbf| \Qbf) \triangleq \sum_{i=1}^n\sum_{j=1}^K P_{ij} \log \frac{P_{ij}}{Q_{ij}}\,.
\end{equation}
The entropic-regularized optimal transport problem (EOT) is expressed as
\begin{equation}
\label{eq:entropic_reg_problem}
    \OT_{\varepsilon}(\abf, \bbf, \Cbf) \triangleq \min_{\Pbf \in U(\abf, \bbf)} \, \langle \Cbf, \Pbf \rangle + \varepsilon \KL(\Pbf|\abf \bbf^\top)\,,
\end{equation}
In \eqref{eq:entropic_reg_problem}, the goal is to find the transport plan that minimizes a trade-off between the transport cost and a measure of ‘‘distance'' to the uniform coupling $\abf \bbf^\top$, which distributes every source point to every target point uniformly.
When $\Pbf$ is a semi-relaxed transport plan the problem \eqref{eq:entropic_reg_problem} will be called a semi-relaxed entropic OT problem.

\section{Maximum-likelihood for mixture models is minimization of EOT}

In this note, we consider a mixture model as described below.
\begin{definition}[(Discrete) mixture model]
\label{def:discrete_mm}
The generative process of a discrete mixture model consists in
\begin{itemize}
\item[$\bullet$] $Y \sim P_{Y}$ where $P_Y = \sum_{j=1}^{K} \pi^\star_j \delta_j$ with $\pibf^\star \in \simplex_K$ represents the discrete distribution on $K$ labels/classes. In other words, $P_Y$ is the distribution of the ‘‘latent variables''.
\item[$\bullet$] $X|y=j \sim P_{X|Y}(\cdot|j, \thetabf^\star)$, where $\thetabf^\star \in \Theta$ and $P_{X|Y}$ is the parametrized distribution of the data given the label.
\end{itemize}
\end{definition}
We note $P_{X,Y}$ the corresponding joint distribution. A simple example of this generative process is the Gaussian mixture model where the parametrized distribution has density $P_{X|Y}(\xbf|j, \thetabf^\star)\propto \exp(-\frac{1}{2}(\xbf-\mubf_j^\star)^\top {\Sigmabf^{\star}}^{-1}(\xbf-\mubf_j^\star))$ where $\mubf^{\star}_j \in \R^d$ is the true mean associated to the $j$-th class and $\Sigmabf^{\star} \succ 0$ the true covariance (assumed to be identical for each class).
In this case $\thetabf^\star = (\mubf^\star_1, \cdots, \mubf^\star_{K}, \Sigmabf^\star)$.

Now suppose that we observe some samples $\xbf_1, \cdots, \xbf_n \sim P_X$ i.i.d. where $P_X$ is the distribution of the data, according to the generative model above.
The goal of maximum-likelihood estimation is to infer the parameters $\pibf^\star, \thetabf^\star$ from these observations.
By independence, the negative log-likelihood for a parameter $\betabf = (\thetabf, \pibf)$ writes
\begin{equation}
\label{eq:first_shot_likelihood}
\begin{split}
\Lcal(\betabf) &= -\sum_{i=1}^{n} \log P_{X}(\xbf_i| \betabf) \\
&\stackrel{\star}{=} -\sum_{i=1}^{n} \log \left(\sum_{j=1}^{K} P_{X,Y}(\xbf_i, j |\betabf)\right)\stackrel{\star \star}{=} -\sum_{i=1}^{n} \log \left(\sum_{j=1}^{K} \pi_j P_{X|Y}(\xbf_i|j, \thetabf) \right) \\
&= -\sum_{i=1}^{n} \log \left(\sum_{j=1}^{K} \pi_j \exp(\log(P_{X|Y}(\xbf_i|j, \thetabf)) \right) \,.
\end{split}
\end{equation}
In $\star$ we used the law of total probability and in $\star \star$ we used the Bayes' formula. In the last line we use a (at first glance) stupid reparametrization.

We will prove three facts: first, the negative log-likelihood \eqref{eq:first_shot_likelihood} can be rewritten as a certain semi-relaxed entropic OT problem; second, there is an entropic OT problem that is an upper bound for the negative log-likelihood; and third, minimizing $\Lcal(\betabf)$ with respect to $\pibf$ results in equality with this upper bound.

The key result to make the connections between OT and log-likelihood is to rewrite the ‘‘logsumexp'' term as a minimization problem over the probability simplex.
 This is next in the following lemma, which is sometimes referred to as the Gibbs variational principle or the dual formulation of the KL divergence.
\begin{restatable}{lemma}{dualkl}
\label{lemma:dual_kl}
Let $\pi_1, \cdots, \pi_K$ be positive real numbers and $h_1, \cdots, h_K \in \R$. Then
\begin{equation*}
\log \left(\sum_{j=1}^{K} \pi_j \exp(h_j)\right) = \max_{\pbf \in \simplex_K} \ \sum_{j=1}^{K} h_j p_j - \sum_{j=1}^{K} p_j \log(\frac{p_j}{\pi_j})\,.
\end{equation*}
The optimal solution is given by $\forall k \in \integ{K}, \ p_k = \frac{\pi_k\exp(h_k)}{\sum_{j=1}^{K}\pi_j\exp(h_j)}$.
\end{restatable}
For now, we postpone the proof of this result (see \Cref{sec:postponed_proof}) but we can use it to rewrite the negative log-likelihood. Combining \eqref{eq:first_shot_likelihood} and \Cref{lemma:dual_kl} with $\log(P_{X|Y}(\xbf_i|j, \betabf))$ in the role of $h_j$ we obtain
\begin{equation*}
\label{eq:second_shot_likelihood}
\begin{split}
\Lcal(\betabf) &= -\sum_{i=1}^{n} \left(\max_{\pbf \in \simplex_K}  \sum_{j=1}^{K} \log(P_{X|Y}(\xbf_i|j, \thetabf)) p_j - \sum_{j=1}^{K} p_j \log(\frac{p_j}{\pi_j})\right) \\
&= -\max_{\pbf^{(1)}, \cdots, \pbf^{(n)} \in \simplex_K} \sum_{i=1}^{n} \sum_{j=1}^K \log(P_{X|Y}(\xbf_i|j, \thetabf)) p^{(i)}_j - \sum_{i=1}^{n} \sum_{j=1}^K p^{(i)}_j \log(\frac{p^{(i)}_j}{\pi_j})\,. \\
\end{split}
\end{equation*}
Equivalently,
\begin{equation*}
\frac{1}{n} \Lcal(\betabf) = \min_{\pbf^{(1)}, \cdots, \pbf^{(n)} \in \simplex_K} \sum_{i,j} C_{ij}(\thetabf) \frac{p^{(i)}_j}{n} + \sum_{i,j} \frac{p^{(i)}_j}{n} \log(\frac{p^{(i)}_j/n}{\pi_j/n})\,,
\end{equation*}
where we introduced $C_{ij}(\thetabf) \triangleq - \log(P_{X|Y}(\xbf_i|j, \thetabf))$.
Now suppose that we have solutions $\pbf^{(1)}, \cdots, \pbf^{(n)}$ of the minimization problem above and that we consider $\Pbf = \frac{1}{n}(\pbf^{(1)}, \cdots, \pbf^{(n)})^\top \in \R_{+}^{n \times K}$.
Then obviously $\Pbf \one_K = \frac{1}{n} \one_n$.
Conversely, any matrix $\Pbf \in \R_{+}^{n \times K}$ with $\Pbf \one_K = \frac{1}{n} \one_n$ can be written as $\Pbf = \frac{1}{n}(\pbf^{(1)}, \cdots, \pbf^{(n)})^\top$ for some probability vectors $\pbf^{(1)}, \cdots, \pbf^{(n)} \in \simplex_K$ (which are simply the rows of $\Pbf$).
In other words, this proves that
\begin{equation}
\label{eq:temp}
\begin{split}
\frac{1}{n} \Lcal(\betabf) &= \min_{\begin{smallmatrix}\Pbf \in \R_{+}^{n \times K} \\ \Pbf \one_K = \frac{1}{n} \one_n \end{smallmatrix}} \sum_{i,j} C_{ij}(\thetabf) P_{ij} + \sum_{i,j} P_{ij}\log(\frac{P_{ij}}{\pi_j/n})\\
&= \min_{\Pbf \in U_{K}(\frac{\one_n}{n})} \langle \Cbf(\thetabf), \Pbf \rangle + \KL(\Pbf| \frac{\one_n}{n} \pibf^\top)\,.
\end{split}
\end{equation}
We almost have the desired entropic OT problem, albeit with a ‘‘semi-relaxed'' constraint instead of a standard coupling constraint: if $\Pbf$ were in $U(\frac{\one_n}{n}, \pibf)$ we would be done.

To obtain an entropic OT problem we only need to rewrite a bit the quantity above: we show that minimizing the negative log-likelihood with respect to
$\pibf$ leads to a coupling constraint rather than a semi-relaxed one.
First, for any admissible $\Pbf \in U_{K}(\frac{\one_n}{n})$ of the minimization problem \eqref{eq:temp}, we have (see \Cref{lemma:kl_terms})
\begin{equation}
\label{eq:kl_div_separation}
\KL(\Pbf|\frac{\one_n}{n} \pibf^\top) = \KL\left(\Pbf|\frac{\one_n}{n}(\Pbf^\top \one_n)^\top\right) + \KL(\Pbf^\top \one_n | \pibf)\,.
\end{equation}
This implies that
\begin{equation}
\label{eq:likelihood_final_form}
\frac{1}{n} \Lcal(\betabf) = \min_{\Pbf \in \begin{smallmatrix}U_{K}(\frac{\one_n}{n}) \end{smallmatrix}} \langle \Cbf(\thetabf), \Pbf \rangle + \KL\left(\Pbf|\frac{\one_n}{n}(\Pbf^\top \one_n)^\top\right) + \KL(\Pbf^\top \one_n | \pibf)\,.
\end{equation}
Now consider $\Pbf^\star$ the solution of the entropic OT problem $\min_{\Pbf \in U(\frac{1}{n} \one_n, \pibf)} \langle \Cbf(\thetabf), \Pbf \rangle  + \KL(\Pbf| \frac{\one_n}{n} \pibf^\top)$.
By the constraints, $\Pbf^\star \in U_{K}(\frac{1}{n} \one_n)$ and $\KL({\Pbf^\star}^{\top}\one_n|\pibf) = 0$. Thus, by suboptimality in \eqref{eq:likelihood_final_form} $\frac{1}{n} \Lcal(\betabf) \leq \langle \Cbf(\thetabf), \Pbf^\star \rangle + \KL\left(\Pbf^\star|\frac{\one_n}{n}({\Pbf^\star}^\top \one_n)^\top\right)$.
Hence we first obtain an upper-bound on the negative log-likelihood:
\begin{equation}
\label{eq:neg_likelihood_smaller}
\frac{1}{n} \Lcal(\betabf) \leq \OT_{\varepsilon=1}(\frac{\one_n}{n}, \pibf, \Cbf(\thetabf))\,.
\end{equation}
To obtain an equality we will minimize with respect to $\pibf$. Using that $\min_{\bbf \in \simplex_K} \KL(\abf|\bbf) = \KL(\abf|\abf) = 0$, if we minimize the RHS in \eqref{eq:likelihood_final_form} over $\pibf$ we get,
\begin{equation*}
\label{eq:min_rhs}
\begin{split}
&\min_{\pibf \in \simplex_K} \min_{\Pbf \in \begin{smallmatrix}U_{K}(\frac{\one_n}{n}) \end{smallmatrix}} \langle \Cbf(\thetabf), \Pbf \rangle + \KL\left(\Pbf|\frac{\one_n}{n}(\Pbf^\top \one_n)^\top\right) + \KL(\Pbf^\top \one_n | \pibf) \\
&=\min_{\Pbf \in \begin{smallmatrix}U_{K}(\frac{\one_n}{n}) \end{smallmatrix}} \langle \Cbf(\thetabf), \Pbf \rangle + \KL\left(\Pbf|\frac{\one_n}{n}(\Pbf^\top \one_n)^\top\right) \stackrel{\star}{=} \min_{\pibf \in \simplex_K}\min_{\Pbf \in U(\frac{1}{n} \one_n, \pibf)} \langle \Cbf(\thetabf), \Pbf \rangle  + \KL(\Pbf| \frac{\one_n}{n} \pibf^\top) \\
&= \min_{\pibf \in \simplex_K} \OT_{\varepsilon=1}(\frac{\one_n}{n}, \pibf, \Cbf(\thetabf))\,.
\end{split}
\end{equation*}
For the equality in $\star$ we used that the second marginal in the optimization problem is redundant in the RHS. 
Precisely, if $\Pbf_0$ is a solution of the LHS then $\bbf_0 = \Pbf_0^\top\one_n \in \simplex_K$ thus $\Pbf_0 \in U(\frac{\one_n}{n}, \bbf_0)$ and $\KL\left(\Pbf_0|\frac{\one_n}{n}(\Pbf_0^\top \one_n)^\top\right) = \KL\left(\Pbf_0|\frac{\one_n}{n}\bbf_0^\top\right)$ which implies that LHS $\geq$ RHS. Conversely, if $(\pibf_1, \Pbf_1)$ is a solution of the RHS problem then $\Pbf_1 \in U_{K}(\frac{\one_n}{n})$ and $\pibf_1 = \Pbf_1^\top \one_n$, hence LHS $\leq$ RHS.

In particular for any $\thetabf \in \Theta, \min_{\pibf} \frac{1}{n}\Lcal(\pibf, \thetabf) = \min_{\pibf} \OT_{\varepsilon=1}(\frac{\one_n}{n}, \pibf, \Cbf(\thetabf))$.
This gives the final result written below.

\begin{proposition}[MLE for mixture models is minimization of an EOT problem]
\label{prop:general_result}
Consider a mixture model as in \Cref{def:discrete_mm} and $\Lcal$ the negative log-likelihood on $n$ i.i.d. samples. First, we have the identity
\begin{equation}
\label{eq:likelihood_almost_final}
\begin{split}
\forall (\pibf, \thetabf) \in \simplex_K \times \Theta, \ \frac{1}{n} \Lcal(\pibf, \thetabf) &= \min_{\Pbf \in U_{K}(\frac{\one_n}{n})} \langle \Cbf(\thetabf), \Pbf \rangle + \KL(\Pbf| \frac{\one_n}{n} \pibf^\top)\,,
\end{split}
\end{equation}
where the cost matrix is $\Cbf(\thetabf)  \triangleq \left(-\log P_{X|Y}(\xbf_i|j, \thetabf)\right)_{ij}$.
Second, the EOT problem is a upper-bound of the negative log-likelihood:
\begin{equation}
\forall (\pibf, \thetabf) \in \simplex_K \times \Theta, \ \frac{1}{n} \Lcal(\pibf, \thetabf) \leq \OT_{\varepsilon=1}(\frac{\one_n}{n}, \pibf, \Cbf(\thetabf))\,.
\end{equation}
Third, minimizing the negative log-likelihood with respect to the parameters is equivalent to minimizing over the parameters an EOT problem. Precisely,
\begin{equation}
\min_{(\pibf, \thetabf) \in \simplex_K \times \Theta} \frac{1}{n} \Lcal(\pibf, \thetabf) = \min_{(\pibf, \thetabf) \in \simplex_K \times \Theta} \OT_{\varepsilon=1}(\frac{\one_n}{n}, \pibf, \Cbf(\thetabf)) \,.
\end{equation}

\end{proposition}

\section{Illustration with Gaussian Mixture Models}

To illustrate the previous results we will show that the updates of the Expectation–Maximization algorithm (EM) for the Gaussian mixture model (GMM) can be interpreted as a block-coordinate descent (BCD) on the EOT loss (see e.g. \citealt[Section 11.4]{murphy2012machine} for a description of the EM algorithm).
According to \Cref{prop:general_result} minimizing the negative log-likelihood is equivalent to solving
\begin{equation}
\label{eq:optim}
\min_{(\pibf, \thetabf) \in \simplex_K \times \Theta} \min_{\Pbf \in U_{K}(\frac{\one_n}{n})} \langle \Cbf(\thetabf), \Pbf \rangle + \KL(\Pbf| \frac{\one_n}{n} \pibf^\top)\,.
\end{equation}
The BCD strategy consists in alternating between minimizing \eqref{eq:optim} in $\pibf, \thetabf, \Pbf$ while keeping the other variables fixed.

The update of $\Pbf$ with $\pibf, \thetabf$ fixed consists in solving a semi-relaxed entropic OT problem. As described e.g. in \citet{flamary2016ost} (or also in the proof of \Cref{prop:general_result}) the problem decouples with respect to the rows of $\Pbf$ and the solution is given by
\begin{equation}
\label{eq:update_P}
\forall (i,j) \in \integ{n} \times \integ{K}, \ P_{ij} = \frac{1}{n} \frac{\pi_j\exp(-C_{ij}(\thetabf))}{\sum_{k=1}^{K}\pi_k \exp(-C_{ik}(\thetabf))} = \frac{1}{n}\frac{\pi_j P_{X|Y}(\xbf|j, \thetabf)}{\sum_{k=1}^{K}\pi_k P_{X|Y}(\xbf|k, \thetabf)}\,.
\end{equation}
This step actually corresponds to finding the conditional distribution in the ‘‘E step'' of the EM algorithm.
The update for $\pibf$, with $\Pbf$ held fixed, can be found using \eqref{eq:kl_div_separation} and is simply
\begin{equation}
\label{update_pi}
\pibf = \Pbf^\top \one_n \text{ i.e. } \forall j \in \integ{K}, \ \pi_j = \sum_{i=1}^{n} P_{ij}\,.
\end{equation}
Finally we derive the update of $\thetabf$ in the GMM case.
We consider
\begin{equation*}
\label{eq:pgmm}
P_{X|Y}(\xbf|j, \thetabf) = (2\pi)^{-d/2} \det \Sigmabf^{-1/2}\exp\left(-\frac{1}{2}(\xbf-\mubf_j)^\top \Sigmabf^{-1}(\xbf-\mubf_j)\right)\,,
\end{equation*}
and the goal is to update $\thetabf = (\mubf_1, \cdots, \mubf_K, \Sigmabf)$ with $\mubf_j \in \R^{d}$ and $\Sigmabf \succ 0$. With other variables fixed this boils down to solving
\begin{equation*}
\min_{(\mubf_1, \cdots, \mubf_K, \Sigmabf)} \ \frac{1}{2}\sum_{ij} (\xbf_i-\mubf_j)^\top \Sigmabf^{-1}(\xbf_i-\mubf_j)P_{ij} + \frac{n}{2} \log \det \Sigmabf\,.
\end{equation*}
Setting the gradient of this loss to zero, one can show that the update of $\Sigmabf$ (with $(\mubf_1, \cdots, \mubf_K)$ fixed) reads
\begin{equation}
\label{eq:update_cov}
\Sigmabf = \frac{1}{n} \sum_{ij} P_{ij} (\xbf_i-\mubf_j)(\xbf_i-\mubf_j)^\top\,,
\end{equation}
and the update of the means (with $\Sigmabf$ fixed) are
\begin{equation}
\label{eq:update_means}
\forall j \in \integ{K}, \ \mubf_j = \frac{1}{\sum_{i}P_{ij}} \sum_{i} P_{ij} \xbf_i\,.
\end{equation}
These updates exactly corresponds to the updates of the EM algorithm apply to a GMM: first update $\Pbf$ according to \eqref{eq:update_P} which is the ‘‘E step'', then the proportion of the classes with \eqref{update_pi} and the means and covariance with \eqref{eq:update_cov} and \eqref{eq:update_means}, which is the ‘‘M step''.

\section{Discussions}

To finish we make a few comments. First, the proof described above can be easily  generalized to infinite mixtures, with appropriate assumptions on $P_Y$, see e.g. \citet[Definition 2]{rigollet2018entropic}. With these assumptions we would obtain that minimizing the negative log-likelihood is equivalent to solving a problem of the form $\inf_{(P_Y, \thetabf)} \OT_{\varepsilon=1}(P_Y, \frac{1}{n} \sum_{i=1}^{n} \delta_{\xbf_i}; \thetabf)$ where the cost of the OT is $c(\xbf, y; \thetabf) = -\log (P_{X|Y}(\xbf|y, \thetabf))$.

Also, as discussed in \citet{mena2020sinkhorn}, the same relations between negative log-likelihood and entropic OT can be obtained for more general generative models where $P_Y$ are $P_X$ are coupled via a joint distribution $Q_{X,Y}^{\thetabf^\star}$ (such as $\dr Q_{X,Y}^{\thetabf^\star}(\xbf, y) = \exp(-g_{\thetabf^\star}(\xbf,y)) \dr P_X(\xbf) \dr P_{Y}(y)$). This setting encompasses the GMM case and the principle of the proof remains similar to the one described here.

\section{Postponed proofs \label{sec:postponed_proof}}
\dualkl*
\begin{proof}
The optimization problem above is a maximization of a strictly concave function. Consider the Lagrangian $L(\pbf, \lambda) = \sum_{j} h_j p_j - \sum_{j} p_j \log(\frac{p_j}{\pi_j}) + \lambda (\sum_j p_j- 1)$. Then $\partial_{p_k} L(\pbf, \lambda) =  h_k - \log(p_k/\pi_k) - 1 + \lambda$ thus $\partial_{p_k} L(\pbf, \lambda) = 0 \iff p_{k} = \exp(\lambda -1) \pi_k \exp(h_k)$. By primal constraints $\sum_{j} p_j = 1 \implies \exp(\lambda-1)= \frac{1}{\sum_j \exp(h_j)\pi_j}$. This gives the desired optimal solution. Also $\forall k \in \integ{K}, \log(p_k/\pi_k) = h_k - \log(\sum_{j=1}^{K}\pi_j\exp(h_j))$ hence $\sum_j p_j \log(p_j/\pi_j) = \sum_j h_j p_j - \log(\sum_{j=1}^{K}\pi_j\exp(h_j)) \sum_j p_j$. Using that $\sum_j p_j = 1$ gives the result.
\end{proof}
We also used the following result about the KL divergence.
\begin{lemma}
\label{lemma:kl_terms}
Let $\Pbf \in \R_+^{n \times K}$ be a matrix and $\abf \in \R^{n}_+, \bbf \in \R^{K}_{+}$ then
\begin{equation}
\KL(\Pbf|\abf \bbf^\top) = \KL\left(\Pbf|(\Pbf \one_K)(\Pbf^\top \one_n)^\top\right) + \KL(\Pbf\one_K|\abf) + \KL(\Pbf^\top \one_n | \bbf)\,.
\end{equation}
\end{lemma}
\begin{proof}
By definition
\begin{equation*}
\begin{split}
\KL(\Pbf|\abf \bbf^\top) &= \sum_{ij} P_{ij} \log(\frac{P_{ij}}{a_ib_j}) = \sum_{ij} P_{ij} \log(\frac{P_{ij}}{(\Pbf \one_k)_i(\Pbf^\top \one_n)_j}\frac{(\Pbf \one_k)_i(\Pbf^\top \one_n)_j}{a_i b_j}) \\
&= \KL\left(\Pbf|(\Pbf \one_K)(\Pbf^\top \one_n)^\top\right) + \sum_{ij} P_{ij} \log(\frac{(\Pbf \one_k)_i}{a_i}) + \sum_{ij} P_{ij} \log(\frac{(\Pbf^\top \one_n)_j}{b_j}) \\
&= \KL\left(\Pbf|(\Pbf \one_K)(\Pbf^\top \one_n)^\top\right) + \sum_{i} (\Pbf \one_k)_i \log(\frac{(\Pbf \one_k)_i}{a_i}) + \sum_{j} (\Pbf^\top \one_n)_j \log(\frac{(\Pbf^\top \one_n)_j}{b_j})\,.
\end{split}
\end{equation*}
\end{proof}

\bibliographystyle{unsrtnat}
\bibliography{references.bib}

\end{document}